\documentclass[11pt]{article}

\usepackage[final]{acl}

\usepackage{times}
\usepackage{latexsym}

\usepackage[T1]{fontenc}

\usepackage[utf8]{inputenc}

\usepackage{microtype}

\usepackage{inconsolata}

\usepackage[utf8]{inputenc} 
\usepackage[T1]{fontenc}    
\usepackage{hyperref}       
\usepackage{url}            
\usepackage{booktabs}       
\usepackage{amsfonts}       
\usepackage{nicefrac}       
\usepackage{microtype}      
\usepackage{multirow}
\usepackage{tabularx}
\usepackage[table]{xcolor}
\usepackage{graphicx}
\usepackage{enumitem}
\usepackage{amsmath}
\usepackage{amssymb}
\usepackage{amsthm}

\newtheorem{proposition}{Proposition}
\newtheorem{lemma}{Lemma}

\usepackage{algorithm}
\usepackage{algorithmic}

\usepackage[most]{tcolorbox}
\newtcolorbox{reformulationbox}{
	colback=gray!4,
	colframe=gray!45,
	boxrule=0.6pt,
	arc=2pt,
	left=6pt,
	right=6pt,
	top=6pt,
	bottom=6pt,
	before skip=8pt,
	after skip=8pt
}
\definecolor{thm}{RGB}{119, 228, 200}
\definecolor{lemma}{RGB}{69, 53, 193}
\newtcolorbox{thmbox}[1][]{colback=thm!5!white,colframe=thm!60!black,boxsep=-4pt,grow to left by=4pt,left=10pt,grow to right by=4pt,right=10pt,top=10pt,bottom=10pt,#1}
\newtcolorbox{lemmabox}[1][]{colback=lemma!5!white,colframe=lemma!60!black,boxsep=-4pt,grow to left by=4pt,left=10pt,grow to right by=4pt,right=10pt,top=10pt,bottom=10pt,#1}

\newcommand{\mathsize}[2][\fontsize{10pt}{12pt}\selectfont]{%
	\mbox{#1$\displaystyle #2$}%
}

\title{Revisiting Reinforcement Learning with Verifiable Rewards from a Contrastive Perspective}

\author{
	Feng Zhang$^{1,2}$ \quad
	Xinhong Ma$^{2}$ \quad
	Ziqiang Dong$^{2,\dagger}$ \quad
	Xi Leng$^{2,3}$ \\
	\textbf{Jianfei Zhao}$^{1,4}$ \quad
	\textbf{Xin Sun}$^{1,\dagger}$ \quad
	\textbf{Yang Yang}$^{2}$ \quad
	\textbf{Guanjun Jiang}$^{2}$ \\
	\vspace{-0.8em}
	\\
	$^{1}$Beijing Institute of Technology \quad
	$^{2}$Qwen Business Unit of Alibaba  \\
	$^{3}$The Chinese University of Hong Kong, Shenzhen \quad
	$^{4}$Zhongguancun Academy \\
	\texttt{\{bit\_zhangfeng, zhqingan, sunxin\}@bit.edu.cn} \\
	\texttt{\{xinhong.mxh, ziqiang.dzq, chris.yang, guanj.jianggj\}@alibaba-inc.com} \\
	\texttt{xileng@link.cuhk.edu.cn}
}

\begin{document}
	\maketitle
	
	\begingroup
	\renewcommand{\thefootnote}{\fnsymbol{footnote}}
	\footnotetext[2]{Corresponding authors.}
	\endgroup
	\begin{abstract}
		Group Relative Policy Optimization (GRPO) is one of the most widely adopted RLVR algorithms for post-training large language models on reasoning tasks.
		We first show that GRPO admits an equivalent discriminative reformulation, in which policy optimization maximizes the expected score gap between verified positive and negative rollouts.
		This reformulation reveals two objective-level limitations: \emph{likelihood-misaligned surrogate scores}, in which clipped ratio-based scores are optimized rather than the sequence likelihoods that govern generation, and \emph{score-insensitive credit assignment}, in which rollout-level credit does not reflect the current score gaps between positive and negative rollouts.
		To address these limitations, we propose ConSPO, a \textbf{Con}trastive \textbf{S}equence-level \textbf{P}olicy \textbf{O}ptimization method that uses length-normalized sequence log-probabilities as rollout scores and contrasts verified positive rollouts against negative distractors within the same group.
		ConSPO optimizes a group-wise InfoNCE-style objective to adaptively strengthen updates for poorly separated positives and high-scoring negatives, together with a curriculum-scheduled margin that preserves separation pressure as training progresses.
		Experiments across diverse settings show that ConSPO outperforms strong baselines on challenging reasoning benchmarks.
	\end{abstract}
	
	\section{Introduction}
	\begin{figure*}[t]
		\centering
		\begin{thmbox}[width=1.0\textwidth]
			\small
			\textbf{Equivalent discriminative formulation.}
			$$
			\mathcal J_0(\theta)
			=
			\mathbb E_q
			\left[
			\frac{1}{2}
			\mathbb E_{\bar{o}\sim \pi_{\rm old}(\cdot|q)}
			\big[
			|A(\bar{o}|q)|
			\big]
			\,
			\mathbb E_{o\sim\widetilde{\pi}_q^+,\,o'\sim\widetilde{\pi}_q^-}
			\big[
			s_\theta^+(o,q)-s_\theta^-(o',q)
			\big]
			\right].
			$$
			
			\textbf{Specialization under binary rewards.}
			$$
			\mathcal J_0(\theta)
			=
			\mathbb E_q
			\sqrt{p(q)(1-p(q))}
			\,
			\mathbb E_{o\sim\pi_{\rm old}^+(\cdot|q),\,o'\sim\pi_{\rm old}^-(\cdot|q)}
			\big[
			s_\theta^+(o,q)-s_\theta^-(o',q)
			\big].
			$$
			
			\footnotesize
			where $A(o|q)$ denotes the GRPO group-relative advantage.
			The distributions $\widetilde{\pi}_q^+$ and $\widetilde{\pi}_q^-$ are obtained by reweighting rollouts from $\pi_{\rm old}(\cdot|q)$ using the positive and negative parts of $A(o|q)$, respectively; under binary rewards, they reduce to $\pi_{\rm old}^+(\cdot|q)$ and $\pi_{\rm old}^-(\cdot|q)$, which are the conditional distributions induced by $\pi_{\rm old}$ over verified positive ($r=1$) and negative ($r=0$) rollouts.
			The quantity $p(q)=\mathbb E_{o\sim\pi_{\rm old}(\cdot|q)}[r(o,q)]$ is the expected reward under $\pi_{\rm old}$, i.e., the pass rate under binary rewards, and $s_\theta^+$ and $s_\theta^-$ are sequence-level averages of clipped token-level importance sampling ratios.
		\end{thmbox}
		\caption{\textbf{Equivalent discriminative view of GRPO.}
			In general, the GRPO objective can be written as a weighted positive-negative discrimination objective, with rollout distributions induced by the positive and negative parts of the relative advantage and scores given by averages of clipped importance ratios.
			In the common RLVR setting with binary rewards, the weighting term simplifies to $\sqrt{p(q)(1-p(q))}$, a factor determined by the group pass rate $p(q)$.}
		\label{fig:grpo-discrimination}
	\end{figure*}
	Reinforcement learning with verifiable rewards (RLVR) has become a standard paradigm for post-training LLMs on reasoning tasks.
	Among RLVR algorithms, Group Relative Policy Optimization (GRPO)~\citep{shao2024deepseekmath, guo2025deepseek} is widely adopted for its critic-free design, which estimates group-relative advantages from outcome rewards within sampled rollout groups.
	Recent studies improve GRPO by modifying importance sampling to mitigate entropy collapse~\citep{zhao2026geometricmean,zheng2025groupsequencepolicyoptimization,gao2025soft}, shaping rewards to control reasoning lengths~\citep{huang2026hapo,yi2025shorterbetter,tan2025towards}, or extending the framework to agentic applications~\citep{jin2025searchr,chen2026research,zhang2026prunetirinferencetimetoolpruning}.
	Despite these advances, existing approaches mainly adjust GRPO around its original objective, leaving the core optimization mechanism underexplored.
	
	Recent studies have highlighted the potential role of discriminative learning principles~\citep{su2025trust,cui2026clipo}, providing a useful departure point from the standard group-relative advantage formulation.
	We make this connection explicit by rewriting GRPO as a weighted positive-negative discrimination objective (Proposition~\ref{proposition:1}), whose general form and binary-reward specialization are summarized in Figure~\ref{fig:grpo-discrimination}.
	Under this view, GRPO increases the sequence-level scores of verified positive rollouts and decreases those of verified negative rollouts, where the scores are averages of clipped token-level importance sampling ratios.
	This reveals \emph{likelihood-misaligned surrogate scores}: GRPO separates positive and negative rollouts using scores derived from clipped importance ratios, rather than the sequence likelihoods that determine generation.
	Consequently, GRPO can increase the surrogate score gap between positive and negative rollouts without improving their likelihood ordering under the current policy.
	This training-inference mismatch weakens the connection between objective improvement and generation quality, leading to suboptimal policy optimization~\citep{meng2024simpo}.
	
	The binary-reward specialization further shows that GRPO reduces to positive-negative discrimination scaled by the group pass rate $p(q)$ through the factor $\sqrt{p(q)(1-p(q))}$.
	This pass-rate weighting exposes \emph{score-insensitive credit assignment}: GRPO determines rollout-level credit based on reward-based advantages tied to the group pass rate, rather than on current score gaps between positive and negative rollouts.
	As detailed by the gradient analysis in Lemma~\ref{lemma:grpo-gradient} and illustrated in Figure~\ref{fig:conspo-gradient}, this mechanism assigns the same coefficient to all positive rollouts in a group and another coefficient to all negative rollouts.
	The resulting update is insensitive to within-group score gaps, treating a positive rollout close to negatives and a high-scoring hard negative equivalently to other rollouts with the same reward.
	Dynamic sampling mitigates this issue by retaining groups with intermediate pass rates~\citep{yu2025dapo,bae2026online,jiang2025vcrl}, but it does not remove the underlying source in the GRPO objective.
	
	\begin{figure*}[t]
		\centering
		\includegraphics[width=1.0\linewidth]{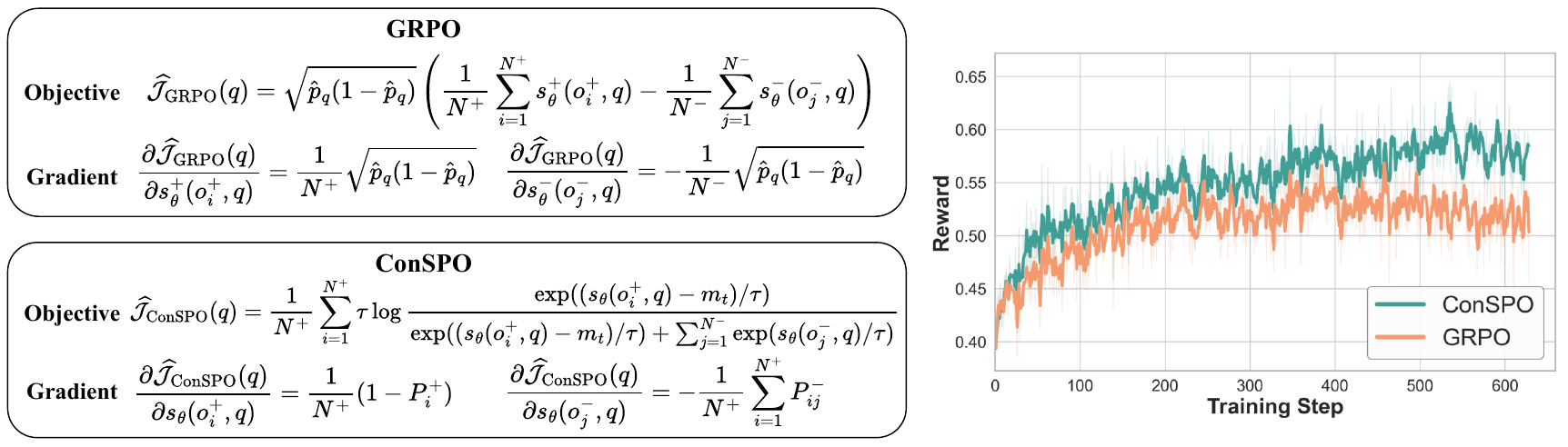}
		\caption{
			\textbf{Left: sequence-level gradient comparison between GRPO and ConSPO.}
			In the GRPO block, $\widehat p_q$ denotes the empirical pass rate of the sampled group, indicating that the credit assigned to positive and negative rollouts is determined by group-level statistics rather than within-group score gaps.
			In the ConSPO block, $P_i^+$ and $P_{ij}^-$ are computed as in Eq.~\eqref{eq:conspo-prob} with $s_\theta(o_i^+,q)$ replaced by $s_\theta(o_i^+,q)-m_t$; the resulting credit assignment depends on current within-group score comparisons.
			\textbf{Right: training dynamics.}
			The curves show that ConSPO achieves consistently higher rewards than GRPO during training.
		}
		\label{fig:conspo-gradient}
	\end{figure*}
	
	To address these limitations, we propose ConSPO, a \textbf{Con}trastive \textbf{S}equence-level \textbf{P}olicy \textbf{O}ptimization method for RLVR.
	ConSPO addresses \emph{likelihood-misaligned surrogate scores} by replacing the clipped importance sampling ratio scores in GRPO with length-normalized sequence log-probabilities.
	This likelihood-based score directly matches the sequence probabilities that govern autoregressive generation, aligning training with inference.
	To address \emph{score-insensitive credit assignment}, ConSPO optimizes a group-wise InfoNCE-style objective~\citep{oord2018representation}, contrasting each verified positive rollout with negative distractors generated for the same question.
	As shown in Lemma~\ref{lemma:conspo-gradient} and Figure~\ref{fig:conspo-gradient}, the softmax normalization makes the positive-rollout gradient depend on its separation from the negative distractors: the update is strengthened when the positive score remains close to the distractors and smoothly attenuated once it becomes well separated.
	For negative rollouts, the total suppressive signal increases when positives are poorly separated from the distractor set, with larger updates assigned to higher-scoring incorrect rollouts.
	ConSPO further introduces a margin to maintain separation pressure when positive rollouts only marginally outperform negative distractors.
	We schedule this margin as a curriculum~\citep{bengio2009curriculum,wang2021survey}, allowing optimization to begin with coarse positive-negative ordering and gradually shift toward stronger separation.
	Experiments across multiple backbone models, model scales, and training datasets show that ConSPO consistently improves performance on seven challenging mathematical reasoning benchmarks.
	Our main contributions are summarized as follows:
	\begin{itemize}[leftmargin=*]
		\item We derive an equivalent discriminative reformulation of GRPO that reveals two objective-level limitations: \emph{likelihood-misaligned surrogate scores} and \emph{score-insensitive credit assignment}.
		
		\item We propose ConSPO, which optimizes an InfoNCE-style sequence-level objective with likelihood-aligned scores and a curriculum-scheduled margin to address the above limitations.
		
		\item Experiments on challenging reasoning benchmarks show that ConSPO consistently outperforms strong RLVR baselines across different models, model scales, and training datasets.
	\end{itemize}
	\section{Related Work}
	\label{sec:related}
	\paragraph{RLVR for LLM reasoning.}
	GRPO~\citep{guo2025deepseek} estimates group-relative advantages from outcome rewards without a learned value model, motivating extensive efforts in RLVR for reasoning.
	Recent work improves GRPO by reallocating losses around the objective~\citep{liu2025understanding,yu2025dapo}, adjusting importance sampling mechanisms for more stable optimization~\citep{zhao2026geometricmean,chen2025minimax,gao2025soft}, or shaping rewards toward more efficient reasoning~\citep{huang2026hapo,yi2025shorterbetter}.
	DisCO~\citep{li2025disco} reinterprets GRPO through a discriminative lens under binary rewards and explores alternative scoring functions, but does not address the score-insensitive credit assignment induced by the GRPO objective.
	ConSPO instead targets this limitation by optimizing an InfoNCE-style objective for online RLVR, with likelihood-aligned sequence scores and a curriculum-scheduled margin.
	\paragraph{Contrastive objectives for policy alignment.}
	Contrastive learning distinguishes positive samples from negative distractors, commonly through InfoNCE-style objectives~\citep{oord2018representation}, with representative methods enlarging negative sets via large-batch sampling or momentum queues~\citep{chen2020simple,he2020momentum}.
	Recent policy alignment methods adopt similar discriminative principles to optimize policy likelihoods over preferred and dispreferred responses without explicit reward modeling~\citep{rafailov2023direct,xu2024contrastive,meng2024simpo}.
	Unlike these offline preference optimization methods, ConSPO transfers the contrastive principle to online RLVR, where positive-negative comparisons are constructed from verifiable rewards on policy-generated rollouts rather than from static pairwise preference data.
	Detailed related work is provided in Appendix~\ref{app:related}.
	\section{Preliminaries}
	\label{sec:pre}
	Given a query $q\in Q$, GRPO samples $G$ rollouts
	$\mathbf{o}=\{o_i\}_{i=1}^G$ from the old policy
	$\pi_{\rm old}(\cdot\mid q)$. Each rollout
	$o_i=(o_{i,1},\ldots,o_{i,|o_i|})$ is generated autoregressively and receives
	a verifiable sequence-level reward $r_i=r(o_i,q)$. The group-relative advantage
	is computed as
	$\widehat A_i=(r_i-\widehat\mu_q)/\widehat\sigma_q$, where
	\begin{equation*}
		\widehat{\mu}_q
		=
		\frac{1}{G}\sum_{j=1}^G r_j,
		\qquad
		\widehat{\sigma}_q^2
		=
		\frac{1}{G}\sum_{j=1}^G
		(r_j-\widehat{\mu}_q)^2 .
		\label{eq:grpo-advantage}
	\end{equation*}
	
	For the analysis in Section~\ref{sec:methods}, we also use the corresponding expected advantage under
	$\pi_{\rm old}(\cdot\mid q)$. Let $\mu_q$ and $\sigma_q^2$ denote the mean and
	variance of $r(o,q)$ over $o\sim\pi_{\rm old}(\cdot\mid q)$. Then
	\begin{equation}
		A(o\mid q)
		=
		\frac{r(o,q)-\mu_q}{\sigma_q}.
		\label{eq:expected-advantage}
	\end{equation}
	
	Subsequently, the policy parameters $\theta$ are updated by maximizing a clipped
	surrogate objective, augmented with a KL divergence penalty to prevent severe
	deviation from a reference model $\pi_{\rm ref}$:
	\begin{equation}
		\mathsize{
			\begin{aligned}
				&\mathcal{J}_{\rm GRPO}(\theta)
				=
				\mathbb{E}_{q\sim Q,\,\{o_i\}_{i=1}^G\sim\pi_{\rm old}(\cdot\mid q)}
				\\
				&\Bigg[
				\frac{1}{G}\sum_{i=1}^G
				\frac{1}{|o_i|}
				\sum_{t=1}^{|o_i|}
				f\!\left(\rho_{i,t}(\theta), \widehat{A}_i\right)
				-
				\beta D_{\rm KL}(\pi_\theta \| \pi_{\rm ref})
				\Bigg].
			\end{aligned}
		}
		\label{eq:grpo-objective}
	\end{equation}
	The token-level importance sampling ratio $\rho_{i,t}(\theta)$ and clipping function $f(\rho,A)$ are defined as
	\begin{equation}
		\begin{aligned}
			\rho_{i,t}(\theta)
			&=
			\frac{
				\pi_\theta(o_{i,t}\mid q,o_{i,<t})
			}{
				\pi_{\rm old}(o_{i,t}\mid q,o_{i,<t})
			},
			\\
			f(\rho,A)
			&=
			\min\!\left(
			\rho A,\,
			\operatorname{clip}(\rho,1-\epsilon,1+\epsilon)A
			\right).
		\end{aligned}
		\label{eq:grpo-clip}
	\end{equation}
	\section{Methods}
	\label{sec:methods}
	In this section, we first reformulate GRPO as a weighted positive-negative discrimination objective, exposing its likelihood-misaligned surrogate scores and score-insensitive credit assignment.
	We then propose ConSPO, which optimizes
	a group-wise InfoNCE-style objective with likelihood-aligned sequence scores and incorporates a curriculum-scheduled margin to maintain separation pressure.
	\subsection{An Equivalent Discriminative Reformulation of GRPO}
	\label{subsec:discriminative-perspective}
	Let $\mathcal J_{\rm GRPO}^{\rm clip}(\theta)$ denote the clipped policy surrogate in expectation form, obtained from Eq.~\eqref{eq:grpo-objective} by removing the KL regularizer.
	Using the expected advantage $A(o\mid q)$ defined in Eq.~\eqref{eq:expected-advantage}, we show that this surrogate admits an equivalent discriminative reformulation.
	\begin{proposition}
		\label{proposition:1}
		For queries with $\sigma_q>0$, the clipped GRPO surrogate $\mathcal J_{\rm GRPO}^{\rm clip}(\theta)$ admits the following equivalent discriminative form:
		\begin{equation}
			\begin{aligned}
				&\mathcal{J}_{\rm GRPO}^{\rm clip}(\theta)
				=
				\mathbb{E}_q
				\Bigg[
				\frac{1}{2}
				\mathbb{E}_{\bar{o}\sim \pi_{\rm old}(\cdot \mid q)}
				\big[
				|A(\bar{o} \mid q)|
				\big]
				\\[-0.8em]
				&
				\mathbb{E}_{o\sim\widetilde{\pi}_q^+,\,o'\sim\widetilde{\pi}_q^-}
				\big[
				s_\theta^+(o,q)-s_\theta^-(o',q)
				\big]
				\Bigg].
				\label{eq:grpo-general-reformulation}
			\end{aligned}
		\end{equation}
		\par
		\noindent where $\widetilde{\pi}_q^+$ and $\widetilde{\pi}_q^-$ denote the normalized reweightings of $\pi_{\rm old}(\cdot\mid q)$ by the positive and negative parts of $A(o\mid q)$, respectively. In the common RLVR setting with binary rewards $r(o,q)\in\{0,1\}$ and valid queries satisfying $0<p(q)<1$, Eq.~\eqref{eq:grpo-general-reformulation} becomes
		\begin{equation}
			\begin{aligned}
				&\mathcal{J}_{\rm GRPO}^{\rm clip}(\theta)
				=
				\mathbb{E}_q
				\sqrt{p(q)(1-p(q))}
				\\
				&\mathbb{E}_{o\sim\pi_{\rm old}^+(\cdot \mid q),\,o'\sim\pi_{\rm old}^-(\cdot \mid q)}
				\big[
				s_\theta^+(o,q)-s_\theta^-(o',q)
				\big].
				\label{eq:grpo-binary-reformulation}
			\end{aligned}
		\end{equation}
		where $p(q)=\mathbb{E}_{o\sim\pi_{\rm old}(\cdot\mid q)}[r(o,q)]$ is the expected reward, i.e., the pass rate, under $\pi_{\rm old}$, and $\pi_{\rm old}^+(\cdot\mid q)$ and $\pi_{\rm old}^-(\cdot\mid q)$ denote $\pi_{\rm old}$ conditioned on verified positive and negative outcomes, respectively.
	\end{proposition}
	
	The proof is provided in Appendix~\ref{app:grpo-reformulation}.
	The sequence-level surrogate scores $s_\theta^+$ and $s_\theta^-$ in Proposition~\ref{proposition:1} are obtained by averaging the clipped token-level importance sampling ratios:
	\begin{equation}
		\mathsize{
			\begin{aligned}
				s_\theta^+(o,q) 
				&= 
				\frac{1}{|o|}
				\sum_{t=1}^{|o|}
				\min\left(
				\frac{\pi_\theta(o_t \mid q,o_{<t})}
				{\pi_{\rm old}(o_t \mid q,o_{<t})},
				1+\epsilon
				\right),
				\\
				s_\theta^-(o',q) 
				&= 
				\frac{1}{|o'|}
				\sum_{t=1}^{|o'|}
				\max\left(
				\frac{\pi_\theta(o'_t \mid q,o'_{<t})}
				{\pi_{\rm old}(o'_t \mid q,o'_{<t})},
				1-\epsilon
				\right).
			\end{aligned}
		}
		\label{eq:score-positive-negative}
	\end{equation}
	To make rollout-level credit assignment explicit, we instantiate the binary reformulation in Eq.~\eqref{eq:grpo-binary-reformulation} on a sampled group with $N^+>0$ and $N^->0$, and differentiate the resulting empirical objective with respect to the rollout scores.
	\begin{lemma}
		\label{lemma:grpo-gradient}
		Consider a sampled valid group with a positive rollout set 
		$\mathcal P_q=\{o_i^+\}_{i=1}^{N^+}$ and a negative rollout set 
		$\mathcal N_q=\{o_j^-\}_{j=1}^{N^-}$, and let 
		$\widehat p_q=N^+/(N^++N^-)$ denote its empirical pass rate.
		Replacing the conditional expectations in Eq.~\eqref{eq:grpo-binary-reformulation} with empirical averages gives
		\begin{equation}
			\begin{aligned}
				&\widehat{\mathcal J}_{\rm GRPO}^{\rm clip}(q)
				=
				\sqrt{\widehat p_q(1-\widehat p_q)}
				\\
				&\left(
				\frac{1}{N^+}\sum_{i=1}^{N^+} s_\theta^+(o_i^+,q)
				-
				\frac{1}{N^-}\sum_{j=1}^{N^-} s_\theta^-(o_j^-,q)
				\right).
				\label{eq:grpo-empirical-discriminative}
			\end{aligned}
		\end{equation}
		The derivatives with respect to the positive and negative rollout scores are
		\begin{equation*}
			\begin{aligned}
				\frac{\partial \widehat{\mathcal J}_{\rm GRPO}^{\rm clip}(q)}
				{\partial s_\theta^+(o_i^+,q)}
				&=
				\frac{1}{N^+}
				\sqrt{\widehat p_q(1-\widehat p_q)},
				\\
				\frac{\partial \widehat{\mathcal J}_{\rm GRPO}^{\rm clip}(q)}
				{\partial s_\theta^-(o_j^-,q)}
				&=
				-
				\frac{1}{N^-}
				\sqrt{\widehat p_q(1-\widehat p_q)}.
			\end{aligned}
			\label{eq:grpo-grad-pos-neg}
		\end{equation*}
	\end{lemma}
	
	The proof of Lemma~\ref{lemma:grpo-gradient} follows by direct differentiation and is provided in Appendix~\ref{app:grpo-reformulation}.
	\paragraph{Implications.}
	The reformulation identifies two objective-level limitations of GRPO.
	First, Eq.~\eqref{eq:score-positive-negative} reveals \emph{likelihood-misaligned surrogate scores}: the positive-negative separation in GRPO is measured by sequence-level averages of clipped importance ratios, rather than by the sequence likelihoods that govern generation.
	Consequently, enlarging the surrogate score gap in Eq.~\eqref{eq:grpo-binary-reformulation} does not necessarily improve the likelihood ordering between positive and negative rollouts under the current policy.
	Second, Lemma~\ref{lemma:grpo-gradient} reveals \emph{score-insensitive credit assignment}: rollout-level credit is determined by coefficients tied to the empirical pass rate $\sqrt{\widehat p_q(1-\widehat p_q)}$, rather than by current positive-negative score gaps within the group.
	As a result, all positives share the same update coefficient and all negatives share another, making the objective insensitive to whether a positive remains poorly separated from negatives or whether a negative is a high-scoring hard distractor.
	\subsection{A Contrastive View of RLVR Optimization}
	\label{subsec:contrastive-rlvr}
	Motivated by the limitations above, ConSPO uses likelihood-aligned sequence scores and group-wise positive-negative contrasts.
	Specifically, it scores each rollout by the length-normalized log-probability $s_\theta(o,q)=\frac{1}{|o|}\sum_{t=1}^{|o|}\log \pi_\theta(o_t\mid q,o_{<t})$.
	Unlike the clipped surrogate scores $s_\theta^+$ and $s_\theta^-$ in Eq.~\eqref{eq:score-positive-negative}, the same scoring function is applied to both positive and negative rollouts, placing all rollouts on the likelihood scale that governs autoregressive generation.
	Given the rollout sets $\mathcal P_q$ and $\mathcal N_q$, ConSPO treats negative rollouts as distractors and optimizes the following InfoNCE-style objective:
	\begin{equation}
		\mathsize{
			\begin{aligned}
				&\widehat{\mathcal J}_{\rm NCE}(q)
				=
				\frac{1}{N^+}
				\sum_{i=1}^{N^+}
				\tau
				\\
				&\log\frac{
					\exp(s_\theta(o_i^+,q)/\tau)
				}{
					\exp(s_\theta(o_i^+,q)/\tau)
					+
					\sum_{j=1}^{N^-}
					\exp(s_\theta(o_j^-,q)/\tau)
				}.
				\label{eq:conspo-infonce}
			\end{aligned}
		}
	\end{equation}
	where $\tau>0$ is the temperature parameter.
	Unlike the linear positive-negative score difference in Eq.~\eqref{eq:grpo-binary-reformulation}, Eq.~\eqref{eq:conspo-infonce} compares each positive rollout against the negative set via softmax normalization.
	For each positive rollout $o_i^+$, we define the contrastive probabilities assigned to itself and to its negative distractors as
	\begin{equation}
		\mathsize{
			\begin{aligned}
				P_i^+
				&=
				\frac{
					\exp(s_\theta(o_i^+,q)/\tau)
				}{
					\exp(s_\theta(o_i^+,q)/\tau)
					+
					\sum_{k=1}^{N^-}
					\exp(s_\theta(o_k^-,q)/\tau)
				},
				\\
				P_{ij}^-
				&=
				\frac{
					\exp(s_\theta(o_j^-,q)/\tau)
				}{
					\exp(s_\theta(o_i^+,q)/\tau)
					+
					\sum_{k=1}^{N^-}
					\exp(s_\theta(o_k^-,q)/\tau)
				}.
			\end{aligned}
		}
		\label{eq:conspo-prob}
	\end{equation}
	
	By construction, $P_i^+ + \sum_{j=1}^{N^-}P_{ij}^- = 1$ for each positive rollout $o_i^+$.
	These probabilities make rollout-level credit depend on within-group scores, as formalized in Lemma~\ref{lemma:conspo-gradient}.
	\begin{lemma}
		\label{lemma:conspo-gradient}
		For the InfoNCE-style objective in Eq.~\eqref{eq:conspo-infonce}, the derivatives with respect to the positive and negative rollout scores are
		\begin{equation}
			\begin{aligned}
				\frac{\partial \widehat{\mathcal J}_{\rm NCE}(q)}
				{\partial s_\theta(o_i^+,q)}
				&=
				\frac{1}{N^+}
				(1-P_i^+),
				\\
				\frac{\partial \widehat{\mathcal J}_{\rm NCE}(q)}
				{\partial s_\theta(o_j^-,q)}
				&=
				-
				\frac{1}{N^+}
				\sum_{i=1}^{N^+}
				P_{ij}^-.
			\end{aligned}
			\label{eq:conspo-gradient}
		\end{equation}
		Moreover, the aggregate derivative over the negative rollout scores and the total derivative over the group satisfy
		\begin{equation}
			\begin{aligned}
				&
				\sum_{j=1}^{N^-}
				\frac{\partial \widehat{\mathcal J}_{\rm NCE}(q)}
				{\partial s_\theta(o_j^-,q)}
				=
				-
				\frac{1}{N^+}
				\sum_{i=1}^{N^+}
				(1-P_i^+),
				\\
				&
				\sum_{i=1}^{N^+}
				\frac{\partial \widehat{\mathcal J}_{\rm NCE}(q)}
				{\partial s_\theta(o_i^+,q)}
				+
				\sum_{j=1}^{N^-}
				\frac{\partial \widehat{\mathcal J}_{\rm NCE}(q)}
				{\partial s_\theta(o_j^-,q)}
				=
				0.
			\end{aligned}
			\label{eq:conspo-gradient-aggregate}
		\end{equation}
		For any two negative rollouts $o_j^-$ and $o_k^-$, the magnitudes of their negative-score derivatives satisfy
		\begin{equation}
			\begin{aligned}
				&
				\frac{
					\left|
					\partial \widehat{\mathcal J}_{\rm NCE}(q) / \partial s_\theta(o_j^-,q)
					\right|
				}{
					\left|
					\partial \widehat{\mathcal J}_{\rm NCE}(q) / \partial s_\theta(o_k^-,q)
					\right|
				}
				=
				\frac{\sum_{i=1}^{N^+}P_{ij}^-}{\sum_{i=1}^{N^+}P_{ik}^-}
				\\
				&=
				\exp\left(
				\frac{s_\theta(o_j^-,q)-s_\theta(o_k^-,q)}{\tau}
				\right).
			\end{aligned}
			\label{eq:conspo-hard-negative}
		\end{equation}
	\end{lemma}
	
	The proof is provided in Appendix~\ref{app:conspo-gradient}.
	Lemma~\ref{lemma:conspo-gradient} shows that ConSPO induces contrast-sensitive credit assignment.
	For positive rollouts, the coefficient $1-P_i^+$ increases when the positive rollout remains close to its negative distractors, and attenuates as it becomes well separated.
	For negative rollouts, Eq.~\eqref{eq:conspo-gradient-aggregate} shows that the aggregate suppressive signal has magnitude
	$\frac{1}{N^+}\sum_{i=1}^{N^+}(1-P_i^+)$, which grows when positives receive low contrastive probabilities.
	This signal is then allocated across negatives according to $P_{ij}^-$, so higher-scoring negatives receive larger derivative magnitudes than lower-scoring ones, as shown in Eq.~\eqref{eq:conspo-hard-negative}.
	The balance relation in Eq.~\eqref{eq:conspo-gradient-aggregate} further shows that positive and negative score derivatives cancel at the group level.
	Compared with Lemma~\ref{lemma:grpo-gradient}, ConSPO assigns rollout-level credit according to current within-group score comparisons, rather than fixed coefficients determined by group reward statistics.
	\begin{table*}[t]
		\centering
		\small
		\setlength{\tabcolsep}{3.5pt}
		\renewcommand{\arraystretch}{1.0}
		\begin{tabularx}{\textwidth}{l*{8}{>{\centering\arraybackslash}X}}
			\toprule
			\textbf{Method}
			& \textbf{AIME24}
			& \textbf{AIME25}
			& \textbf{AIME26}
			& \textbf{HMMT25}
			& \textbf{MATH500}
			& \textbf{AMC}
			& \textbf{O-Bench}
			& \textbf{Avg.} \\
			\midrule
			Base Model & 20.9 & 20.7 & 13.9 & 9.7 & 79.2 & 52.8 & 37.6 & 33.5 \\
			\midrule
			GRPO      & 28.6 & 22.9 & 20.4 & 11.7 & 85.4 & 64.4 & 46.5 & 40.0 \\
			DAPO      & 28.6 & 22.9 & 20.8 & 13.5 & 86.0 & 66.2 & 50.2 & 41.2 \\
			Dr.\ GRPO & 27.4 & 22.0 & 19.5 & 10.7 & 83.4 & 63.5 & 46.7 & 39.0 \\
			DisCO     & 28.3 & 24.8 & 21.7 & \underline{14.4} & 86.6 & 66.5 & 49.3 & 41.7 \\
			GMPO      & \underline{31.5} & 24.6 & \underline{23.8} & 12.8 & 87.0 & 70.0 & 49.0 & 42.7 \\
			CISPO     & 31.4 & 23.0 & 23.0 & 13.3 & \underline{87.2} & 65.2 & 47.6 & 41.5 \\
			SAPO      & 30.6 & \underline{25.3} & 23.3 & 12.8 & \underline{87.2} & \underline{70.1} & \underline{50.4} & \underline{42.8} \\
			\midrule
			\rowcolor{green!15}
			\textbf{ConSPO (Ours)}
			& \textbf{34.7}
			& \textbf{26.7}
			& \textbf{23.9}
			& \textbf{14.9}
			& \textbf{88.8}
			& \textbf{70.4}
			& \textbf{51.1}
			& \textbf{44.4} \\
			\bottomrule
		\end{tabularx}
		\caption{
			\textbf{Comparative performance on seven mathematical reasoning benchmarks with DeepSeek-R1-Distill-Qwen-1.5B.}
			We report avg@32 on AIME, HMMT, and AMC, and pass@1 on the remaining benchmarks.
			O-Bench is short for OlympiadBench.
			The best and runner-up results are highlighted in \textbf{bold} and \underline{underlined}, respectively.
		}
		\label{tab:1}
	\end{table*}
	\begin{table*}[t]
		\centering
		\small
		\setlength{\tabcolsep}{3.5pt}
		\renewcommand{\arraystretch}{1.0}
		\begin{tabularx}{\textwidth}{l*{8}{>{\centering\arraybackslash}X}}
			\toprule
			\textbf{Method}
			& \textbf{AIME24}
			& \textbf{AIME25}
			& \textbf{AIME26}
			& \textbf{HMMT25}
			& \textbf{MATH500}
			& \textbf{AMC}
			& \textbf{O-Bench}
			& \textbf{Avg.} \\
			\midrule
			
			\rowcolor{gray!10}
			\multicolumn{9}{c}{\textit{DeepSeek-R1-Distill-Qwen-7B}} \\
			\midrule
			Base Model & 41.9 & 30.3 & 35.5 & 17.5 & 88.0 & 70.3 & 50.7 & 47.7 \\
			GRPO       & 50.0 & 35.9 & 43.9 & 20.7 & 92.8 & \underline{82.8} & \underline{59.4} & 55.1 \\
			DAPO       & 49.9 & 34.2 & 37.9 & 20.4 & \underline{93.4} & 80.6 & 58.5 & 53.6 \\
			Dr.\ GRPO  & \underline{53.2} & 35.3 & 43.3 & \textbf{21.7} & 92.4 & 82.2 & 59.1 & 55.3 \\
			DisCO      & 51.8 & \underline{36.9} & \underline{45.1} & 19.6 & 93.2 & 82.0 & 59.3 & \underline{55.4} \\
			\midrule
			\rowcolor{green!15}
			\textbf{ConSPO (Ours)}
			& \textbf{54.9}
			& \textbf{39.1}
			& \textbf{46.8}
			& \underline{21.5}
			& \textbf{93.8}
			& \textbf{83.8}
			& \textbf{62.4}
			& \textbf{57.5} \\
			
			\midrule
			\rowcolor{gray!10}
			\multicolumn{9}{c}{\textit{DeepSeek-R1-Distill-Llama-8B}} \\
			\midrule
			Base Model & 28.9 & 21.3 & 20.6 & 13.5 & 81.6 & 63.5 & 42.1 & 38.8 \\
			GRPO       & 42.8 & 25.5 & \underline{34.2} & 18.5 & 88.6 & 78.6 & 55.7 & 49.1 \\
			DAPO       & 40.0 & 24.0 & 29.9 & 17.9 & 87.6 & 77.6 & 53.3 & 47.2 \\
			Dr.\ GRPO  & \textbf{44.8} & 26.8 & 32.4 & 19.6 & \underline{89.2} & \underline{80.5} & \underline{56.1} & \underline{49.9} \\
			DisCO      & \underline{44.2} & \underline{28.8} & 32.8 & \underline{19.7} & 88.8 & 78.4 & 54.7 & 49.6 \\
			\midrule
			\rowcolor{green!15}
			\textbf{ConSPO (Ours)}
			& 43.0
			& \textbf{29.6}
			& \textbf{35.5}
			& \textbf{22.6}
			& \textbf{90.4}
			& \textbf{81.1}
			& \textbf{60.7}
			& \textbf{51.8} \\
			
			\bottomrule
		\end{tabularx}
		\caption{
			\textbf{Comparative performance on larger reasoning models.}
		}
		\label{tab:2}
		\vspace{-0.8em}
	\end{table*}
	\subsection{Scheduled Margin as a Separation Curriculum}
	\label{subsec:scheduled-margin}
	Although the InfoNCE-style objective in Eq.~\eqref{eq:conspo-infonce} enables contrast-sensitive credit assignment, its gradients may attenuate once positive rollouts only marginally outperform their negative distractors.
	This becomes more pronounced in later training steps, where positives and negatives may be correctly ordered but remain insufficiently separated.
	To preserve separation pressure, ConSPO introduces a scheduled margin that lowers the effective score of each positive rollout in the contrastive objective.
	The margin starts from zero, allowing early training to focus on coarse positive-negative ordering, and gradually increases to enforce stronger separation as optimization progresses.
	At training step $t$, ConSPO maximizes the resulting margin-enhanced contrastive objective:
	\begin{equation}
		\mathsize[\fontsize{9pt}{12pt}\selectfont]{
			\begin{aligned}
				&\widehat{\mathcal J}_{\rm ConSPO}(q)
				=
				\frac{1}{N^+}
				\sum_{i=1}^{N^+}
				\tau
				\\
				&\log\frac{
					\exp((s_\theta(o_i^+,q)-m_t)/\tau)
				}{
					\exp((s_\theta(o_i^+,q)-m_t)/\tau)
					+
					\sum_{j=1}^{N^-}
					\exp(s_\theta(o_j^-,q)/\tau)
				} .
			\end{aligned}
		}
		\label{eq:conspo-margin}
	\end{equation}
	where $m_t\ge 0$ is the scheduled margin.
	By subtracting $m_t$ from the positive score, ConSPO requires a positive rollout to exceed its negative distractors by a larger margin before the contrastive gradient attenuates.
	
	Let $\lambda_t\in[0,1]$ denote the normalized training progress, and let $\alpha\in(0,1]$ denote the fraction of training used for margin warmup.
	We define the effective schedule progress and the margin as
	\begin{equation*}
		\rho_t=\min\left(\frac{\lambda_t}{\alpha},1\right),
		\quad
		m_t
		=
		\frac{M}{2}
		\left(1-\cos(\pi\rho_t)\right).
		\label{eq:conspo-margin-schedule}
	\end{equation*}
	where $M$ is the target margin.
	Thus, the margin smoothly increases from zero to $M$ during the first $\alpha$ portion of training and remains fixed afterward.
	
	\section{Experiments}
	\label{sec:exp}
	\begin{table*}[t]
		\centering
		\small
		\setlength{\tabcolsep}{3.5pt}
		\renewcommand{\arraystretch}{1.0}
		\begin{tabularx}{\textwidth}{l*{8}{>{\centering\arraybackslash}X}}
			\toprule
			\textbf{Method} & \textbf{AIME24} & \textbf{AIME25} & \textbf{AIME26} & \textbf{HMMT25} & \textbf{MATH500} & \textbf{AMC} & \textbf{O-Bench} & \textbf{Avg.} \\
			\midrule
			
			Base Model & 9.2 & 5.7 & 4.8 & 0.9 & 49.4 & 31.7 & 27.1 & 18.4 \\
			GRPO       & 16.5 & 10.7 & 10.5 & 2.6 & 80.2 & 48.6 & 43.3 & 30.3 \\
			DAPO       & 14.6 & 10.1 & 8.5 & 2.2 & 82.0 & 48.4 & 43.1 & 29.8 \\
			Dr.\ GRPO  & \underline{17.4} & 11.4 & \underline{12.4} & 2.6 & \underline{82.6} & 51.7 & \underline{45.0} & \underline{31.9} \\
			DisCO      & \textbf{19.8} & \underline{12.2} & 9.4 & \underline{5.3} & \textbf{83.0} & \underline{52.9} & 39.1 & 31.7 \\
			\midrule
			\rowcolor{green!15}
			\textbf{ConSPO (Ours)} & \textbf{19.8} & \textbf{15.6} & \textbf{12.8} & \textbf{8.0} & 82.2 & \textbf{53.8} & \textbf{45.6} & \textbf{34.0} \\
			
			\bottomrule
		\end{tabularx}
		\caption{
			\textbf{Generalization on Qwen3-4B-Base.}
			We report avg@32 on AIME, HMMT, and AMC, and pass@1 on the remaining benchmarks.
			O-Bench is short for OlympiadBench.
			The best and runner-up results are highlighted in \textbf{bold} and \underline{underlined}, respectively.
		}
		\label{tab:3}
	\end{table*}
	
	\begin{table*}[t]
		\centering
		\small
		\setlength{\tabcolsep}{3.5pt}
		\renewcommand{\arraystretch}{1.0}
		\begin{tabularx}{\textwidth}{l*{8}{>{\centering\arraybackslash}X}}
			\toprule
			\textbf{Method}
			& \textbf{AIME24}
			& \textbf{AIME25}
			& \textbf{AIME26}
			& \textbf{HMMT25}
			& \textbf{MATH500}
			& \textbf{AMC}
			& \mbox{\textbf{O-Bench}}
			& \textbf{Avg.} \\
			\midrule
			
			Base Model & 20.9 & 20.7 & 13.9 & 9.7 & 79.2 & 52.8 & 37.6 & 33.5 \\
			GRPO       & 29.2 & 22.7 & \underline{21.6} & \underline{13.1} & 83.4 & 64.5 & 44.6 & 39.9 \\
			DAPO       & 29.2 & 22.4 & 19.6 & 12.8 & 83.8 & 65.4 & 49.5 & 40.4 \\
			Dr.\ GRPO  & 30.0 & 22.9 & 19.6 & 12.8 & 84.0 & 65.2 & 45.8 & 40.0 \\
			DisCO      & \underline{31.3} & \underline{24.2} & 20.2 & 12.9 & \underline{86.6} & \underline{69.5} & \underline{50.4} & \underline{42.2} \\
			\midrule
			\rowcolor{green!15}
			\textbf{ConSPO (Ours)}
			& \textbf{34.2}
			& \textbf{25.8}
			& \textbf{24.8}
			& \textbf{14.2}
			& \textbf{87.6}
			& \textbf{70.0}
			& \textbf{52.0}
			& \textbf{44.1} \\
			
			\bottomrule
		\end{tabularx}
		\caption{
			\textbf{Generalization with the DAPO-Math-17k dataset.}
		}
		\label{tab:5}
	\end{table*}
	
	\begin{table}[t]
		\centering
		\small
		\setlength{\tabcolsep}{0.5pt}
		\renewcommand{\arraystretch}{1.0}
		\begin{tabularx}{\columnwidth}{l*{5}{>{\centering\arraybackslash}X}}
			\toprule
			\textbf{Method}
			& \textbf{AIME24}
			& \textbf{MATH}
			& \textbf{AMC}
			& \mbox{\textbf{O-Bench}}
			& \textbf{Avg.} \\
			\midrule
			
			Base Model & 3.3 & 26.4 & 12.5 & 9.3 & 12.9 \\
			GRPO       & \underline{12.0} & 52.8 & 25.9 & \underline{18.7} & 27.4 \\
			DAPO       & 10.5 & 52.8 & 25.0 & 18.4 & 26.7 \\
			Dr.\ GRPO  & 11.1 & 54.0 & 26.5 & \underline{18.7} & 27.6 \\
			DisCO      & 11.4 & \underline{56.6} & \textbf{27.5} & \underline{18.7} & \underline{28.6} \\
			\midrule
			\rowcolor{green!15}
			\textbf{ConSPO}
			& \textbf{12.3}
			& \textbf{57.4}
			& \underline{26.7}
			& \textbf{20.4}
			& \textbf{29.2} \\
			
			\bottomrule
		\end{tabularx}
		\caption{
			\textbf{Generalization on Llama-3.2-3B-Instruct.}
		}
		\label{tab:4}
	\end{table}
	
	\begin{table*}[t]
		\centering
		\small
		\setlength{\tabcolsep}{3.pt}
		\renewcommand{\arraystretch}{1.0}
		\begin{tabularx}{\textwidth}{l*{8}{>{\centering\arraybackslash}X}}
			\toprule
			\textbf{Method}
			& \textbf{AIME24}
			& \textbf{AIME25}
			& \textbf{AIME26}
			& \textbf{HMMT25}
			& \textbf{MATH500}
			& \textbf{AMC}
			& \mbox{\textbf{O-Bench}}
			& \textbf{Avg.} \\
			\midrule
			
			\rowcolor{green!15}
			\textbf{ConSPO}
			& 34.7
			& 26.7
			& 23.9
			& 14.9
			& 88.8
			& 70.4
			& 51.1
			& 44.4 \\
			
			\midrule
			w/o contrastive objective & 32.0 & 24.8 & 21.1 & 13.6 & 85.8 & 69.7 & 52.1 & 42.7 \\
			w/o likelihood-aligned scores & 33.1 & 25.9 & 22.1 & 13.4 & 88.6 & 70.1 & 50.7 & 43.4 \\
			w/o scheduled margin      & 34.3 & 25.4 & 23.4 & 14.2 & 89.0 & 70.6 & 51.7 & 44.1 \\
			w/ fixed margin           & 34.1 & 25.8 & 23.8 & 14.7 & 87.8 & 70.1 & 51.3 & 43.9 \\
			
			\bottomrule
		\end{tabularx}
		\caption{
			\textbf{Ablation study of key components.}
			The variant ``w/o contrastive objective'' replaces the InfoNCE-style objective with the linear positive-negative score difference induced by Eq.~\eqref{eq:grpo-empirical-discriminative}.
			The variant ``w/o likelihood-aligned scores'' replaces the rollout scores defined by length-normalized sequence log-probabilities with clipped importance sampling ratio surrogate scores from Eq.~\eqref{eq:score-positive-negative}.
			The variant ``w/o scheduled margin'' removes the margin, while ``w/ fixed margin'' replaces the scheduled margin with a fixed one.
		}
		\label{tab:6}
	\end{table*}
	\begin{figure*}[ht]
		\centering
		\includegraphics[width=\linewidth]{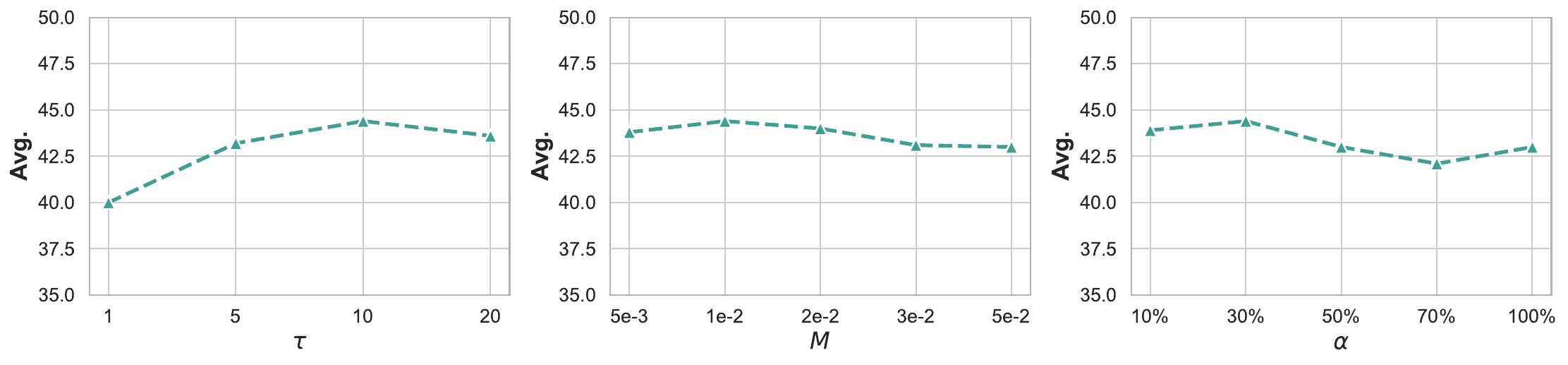}
		\caption{
			\textbf{Parameter study of ConSPO.}
			We vary the contrastive temperature $\tau$, target margin $M$, and margin warmup ratio $\alpha$, and report the average performance on seven reasoning benchmarks.
		}
		\label{fig:para}
	\end{figure*}
	\subsection{Settings}
	\paragraph{Models.} We conduct experiments across diverse model families and scales to validate the generalization of ConSPO, including DeepSeek-R1-Distill-Qwen-1.5B~\citep{guo2025deepseek}, DeepSeek-R1-Distill-Qwen-7B, DeepSeek-R1-Distill-Llama-8B, Qwen3-4B-Base~\citep{yang2025qwen3}, and Llama-3.2-3B-Instruct~\citep{grattafiori2024llama}.
	\paragraph{Training data.} We use the DeepScaleR-Preview-Dataset~\citep{tan2026deepscaler} which comprises approximately 40k mathematical problems drawn from AIME~\citep{zhao2026geometricmean}, AMC~\citep{li2024numinamath}, Omni-MATH~\citep{gao2025omnimath}, and Still~\citep{min2024imitate}. We also use DAPO-Math-17k~\citep{yu2025dapo} to validate the effectiveness of ConSPO on different training data.
	\paragraph{Benchmarks.} We evaluate our models on seven benchmarks, including AIME 2024, AIME 2025, AIME 2026, HMMT 2025~\citep{balunovic_srimatharena_2025}, MATH500~\citep{hendrycks2measuring}, AMC, and OlympiadBench~\citep{he2024olympiadbench}. We report pass@1 on most benchmarks, and avg@32 on AIME, HMMT, and AMC to mitigate the variance introduced by the relatively small size of these datasets~\citep{liu2025ghpo}. More details are provided in Appendix~\ref{Appendix: benchmark}.
	\paragraph{Baselines.} We compare ConSPO against seven recent state-of-the-art RLVR methods, including (1) \textbf{GRPO}; (2) \textbf{DAPO}, which refines the GRPO loss formulation and introduces dynamic sampling and overlong reward shaping; (3) \textbf{Dr.~GRPO}, which removes the standard deviation in advantage estimation and reorganizes loss allocation; (4) \textbf{DisCO}, which revisits GRPO from a discriminative perspective and eliminates its difficulty bias; (5) \textbf{GMPO}, which introduces sequence-level geometric-mean importance sampling; (6) \textbf{CISPO}, which introduces soft clipping in importance sampling; and (7) \textbf{SAPO}, which replaces hard clipping with a smooth, temperature-controlled gate to adaptively control off-policy updates.
	\paragraph{Implementation details.} 
	During training, we set the temperature to 1.0 to ensure sufficient exploration and the learning rate to $2 \times 10^{-6}$. The maximum response length is 8K for DeepSeek-R1-Distill-Qwen and DeepSeek-R1-Distill-Llama, and 4K for all other models. All methods run for 4 epochs, except for Qwen3-4B-Base, which runs for 2 epochs due to its faster convergence.
	Following standard practice~\citep{guo2025deepseek,zhao2026geometricmean}, we discard all-positive and all-negative groups during training, as they do not permit a valid positive-negative contrast.
	We set the contrastive temperature $\tau$ to 10, the target margin $M$ to 0.01, and the margin warmup ratio $\alpha$ to 30\%.
	Evaluations are conducted every 100 steps, with a temperature of 0.6, a top-$p$ of 0.95, and the same maximum response length used during training. The best performance for each method is reported.
	Experiments on the 1.5B, 3B, and 4B models are run on a single compute node, whereas those on the 7B and 8B models are distributed across four compute nodes.
	More details are provided in Appendix~\ref{Appendix: implementation}.
	\subsection{Main Results}
	\paragraph{Performance with reasoning models.}
	As shown in Tables~\ref{tab:1} and~\ref{tab:2}, ConSPO delivers consistent improvements on reasoning models.
	On DeepSeek-R1-Distill-Qwen-1.5B, ConSPO achieves the best average performance of 44.4, outperforming GRPO by 4.4 points and the strongest baseline by 1.6 points.
	Moreover, ConSPO achieves the highest performance across all seven benchmarks in this setting, indicating its generalization across different reasoning tasks.
	The same trend holds for larger reasoning models.
	On DeepSeek-R1-Distill-Qwen-7B and DeepSeek-R1-Distill-Llama-8B, ConSPO improves over the strongest baseline by 2.1 and 1.9 points, respectively.
	These results suggest that the proposed contrastive optimization scheme remains effective across different parameter scales.
	\paragraph{Generalization across backbone models and training datasets.}
	As shown in Tables~\ref{tab:3} and~\ref{tab:4}, ConSPO generalizes well across both base and instruction-tuned models.
	On Qwen3-4B-Base, ConSPO achieves the best average performance of 34.0, outperforming the strongest baseline by 2.1 points.
	Notably, on the challenging HMMT25 benchmark, where the base model obtains only 0.9, ConSPO improves the score to 8.0, suggesting that our method can substantially enhance performance on difficult reasoning tasks where the initial model performs poorly.
	On Llama-3.2-3B-Instruct, whose base reasoning capability is substantially weaker, ConSPO still achieves the highest average performance of 29.2, improving over the strongest baseline by 0.6 points.
	Moreover, Table~\ref{tab:5} shows that ConSPO remains effective when trained on the DAPO-Math-17k dataset, achieving an average performance of 44.1 and outperforming the strongest baseline by 1.9 points.
	These results demonstrate that ConSPO's effectiveness is not tied to a specific backbone or training dataset.
	\subsection{Ablation Study}
	Table~\ref{tab:6} ablates the key components of ConSPO on DeepSeek-R1-Distill-Qwen-1.5B.
	Replacing the InfoNCE-style objective with a linear positive-negative score difference causes the largest performance drop, reducing the average score from 44.4 to 42.7.
	This result shows that group-wise contrastive optimization is central to ConSPO, as it enables credit assignment to adapt to current within-group score comparisons.
	Replacing the length-normalized sequence log-probabilities with clipped importance sampling ratio scores also decreases the average score to 43.4, confirming the importance of aligning rollout scores with the likelihood scale that governs autoregressive generation.
	The scheduled margin further improves performance: removing the margin reduces the average score to 44.1, while using a fixed margin achieves 43.9.
	These results indicate that the contrastive objective, likelihood-aligned sequence scoring, and curriculum-scheduled margin play complementary roles in ConSPO's effectiveness.
	\subsection{Parameter Study}
	Figure~\ref{fig:para} studies the sensitivity of ConSPO to the contrastive temperature $\tau$, target margin $M$, and margin warmup ratio $\alpha$.
	When $\tau$ is too small, the contrastive distribution becomes overly sharp, potentially concentrating the update on a few high-scoring negatives and leading to less stable optimization.
	For the margin schedule, performance drops when either the target margin or the warmup ratio is large, suggesting that ConSPO benefits from a moderate separation requirement introduced at a proper pace.
	
	\section{Conclusion}
	In this paper, we derive an equivalent discriminative reformulation of GRPO as a weighted positive-negative discrimination objective.
	This reformulation exposes two objective-level limitations: \emph{likelihood-misaligned surrogate scores}, where GRPO optimizes clipped ratio-based scores rather than generation likelihoods, and \emph{score-insensitive credit assignment}, where rollout-level credit is insensitive to positive-negative score gaps.
	To address these limitations, we propose ConSPO, which optimizes a group-wise InfoNCE-style objective with likelihood-aligned sequence scores and a curriculum-scheduled margin.
	Extensive experiments across diverse settings validate the effectiveness and generalization of ConSPO.
	
	\section*{Limitations}
	Despite our best efforts, this study has several limitations.
	First, although our experiments cover multiple backbone models, model scales, and training datasets, our empirical validation primarily focuses on models up to the 8B scale.
	While we additionally compare ConSPO and GRPO on a 32B model in Appendix~\ref{app:32b}, more systematic evaluations at this scale and beyond remain an important direction for future work.
	Second, ConSPO is a preliminary exploration of RLVR from a contrastive optimization perspective.
	We instantiate this perspective with an InfoNCE-style group-wise objective, likelihood-aligned sequence scores, and a curriculum-scheduled margin.
	Broader choices of objective formulations and rollout scoring functions remain worth investigating to better understand and generalize contrast-sensitive credit assignment in RLVR.
	
	\section*{Ethical Considerations}
	This work studies reinforcement learning methods to improve LLM reasoning abilities and does not introduce new datasets that contain private, sensitive, or personally identifiable information.
	The potential negative impacts of ConSPO are similar to those associated with general LLM reasoning technologies, including risks from misuse or insufficiently supervised deployment.
	We emphasize the importance of careful evaluation and appropriate safeguards to ensure fair, safe, and responsible use of LLM systems.
	
	
	\section*{Acknowledgement}
	This work was supported by Qwen Business Unit through Alibaba Research Intern Program.
	
	
	\bibliography{custom}
	\appendix
	\onecolumn
	
	\section{Proofs}
	\label{app:proofs}
	\subsection{Proof of Proposition~\ref{proposition:1} and Lemma~\ref{lemma:grpo-gradient}}
	\label{app:grpo-reformulation}
	
	We first state two auxiliary lemmas.
	
	\begin{lemma}
		\label{lemma:advantage-reweighting}
		For any query $q$ with $\sigma_q>0$, define
		\[
		A^+(o\mid q)=\max(A(o\mid q),0),
		\qquad
		A^-(o\mid q)=\max(-A(o\mid q),0).
		\]
		Let
		\[
		W_+(q)=\mathbb E_{o\sim\pi_{\rm old}(\cdot\mid q)}[A^+(o\mid q)],
		\qquad
		W_-(q)=\mathbb E_{o\sim\pi_{\rm old}(\cdot\mid q)}[A^-(o\mid q)].
		\]
		Then $W_+(q)=W_-(q)$. Denoting the common value by $W(q)$, we have
		\[
		W(q)
		=
		\frac{1}{2}
		\mathbb E_{\bar o\sim\pi_{\rm old}(\cdot\mid q)}
		\big[
		|A(\bar o\mid q)|
		\big]
		>0.
		\]
		Moreover,
		\[
		\widetilde{\pi}_q^+(do)
		=
		\frac{A^+(o\mid q)}{W(q)}\,\pi_{\rm old}(do\mid q),
		\qquad
		\widetilde{\pi}_q^-(do)
		=
		\frac{A^-(o\mid q)}{W(q)}\,\pi_{\rm old}(do\mid q)
		\]
		are valid probability distributions.
	\end{lemma}
	
	\begin{proof}
		By the definition of $A(o\mid q)$,
		\begin{align*}
			\mathbb E_{o\sim\pi_{\rm old}(\cdot\mid q)}[A(o\mid q)]
			&=
			\frac{
				\mathbb E_{o\sim\pi_{\rm old}(\cdot\mid q)}[r(o,q)]-\mu_q
			}{\sigma_q}
			=0.
		\end{align*}
		Since $A=A^+-A^-$, it follows that $W_+(q)=W_-(q)$. Since $|A|=A^++A^-$,
		\[
		\mathbb E_{\bar o\sim\pi_{\rm old}(\cdot\mid q)}
		\big[
		|A(\bar o\mid q)|
		\big]
		=
		W_+(q)+W_-(q)
		=
		2W(q).
		\]
		The condition $\sigma_q>0$ implies that $A(o\mid q)$ is not almost surely zero; combined with $\mathbb E[A(o\mid q)]=0$, this gives $W(q)>0$.
		
		The reweighted measures are nonnegative, and
		\[
		\int \widetilde{\pi}_q^+(do)
		=
		\frac{
			\mathbb E_{o\sim\pi_{\rm old}(\cdot\mid q)}[A^+(o\mid q)]
		}{W(q)}
		=1,
		\qquad
		\int \widetilde{\pi}_q^-(do)
		=
		\frac{
			\mathbb E_{o\sim\pi_{\rm old}(\cdot\mid q)}[A^-(o\mid q)]
		}{W(q)}
		=1.
		\]
	\end{proof}
	
	\begin{lemma}
		\label{lemma:binary-specialization}
		Assume binary rewards $r(o,q)\in\{0,1\}$ and define
		\[
		p(q)=\mathbb E_{o\sim\pi_{\rm old}(\cdot\mid q)}[r(o,q)].
		\]
		For any query with $0<p(q)<1$,
		\[
		W(q)=\sqrt{p(q)(1-p(q))},
		\qquad
		\widetilde{\pi}_q^+=\pi_{\rm old}^+(\cdot\mid q),
		\qquad
		\widetilde{\pi}_q^-=\pi_{\rm old}^-(\cdot\mid q),
		\]
		where $\pi_{\rm old}^+(\cdot\mid q)$ and $\pi_{\rm old}^-(\cdot\mid q)$ denote $\pi_{\rm old}$ conditioned on $r(o,q)=1$ and $r(o,q)=0$, respectively.
	\end{lemma}
	
	\begin{proof}
		Under binary rewards,
		\[
		\mu_q=p(q),
		\qquad
		\sigma_q=\sqrt{p(q)(1-p(q))},
		\qquad
		A(o\mid q)
		=
		\begin{cases}
			\sqrt{\dfrac{1-p(q)}{p(q)}}, & r(o,q)=1,\\[8pt]
			-\sqrt{\dfrac{p(q)}{1-p(q)}}, & r(o,q)=0.
		\end{cases}
		\]
		Therefore,
		\[
		W(q)
		=
		\mathbb E_{o\sim\pi_{\rm old}(\cdot\mid q)}[A^+(o\mid q)]
		=
		p(q)\sqrt{\frac{1-p(q)}{p(q)}}
		=
		\sqrt{p(q)(1-p(q))}.
		\]
		Substituting this value into the definitions of $\widetilde{\pi}_q^+$ and $\widetilde{\pi}_q^-$ gives
		\begin{align*}
			\widetilde{\pi}_q^+(do)
			&=
			\frac{A^+(o\mid q)}{W(q)}\,\pi_{\rm old}(do\mid q)
			=
			\frac{\mathbf 1\{r(o,q)=1\}}{p(q)}\,\pi_{\rm old}(do\mid q)
			=
			\pi_{\rm old}^+(do\mid q), \\
			\widetilde{\pi}_q^-(do)
			&=
			\frac{A^-(o\mid q)}{W(q)}\,\pi_{\rm old}(do\mid q)
			=
			\frac{\mathbf 1\{r(o,q)=0\}}{1-p(q)}\,\pi_{\rm old}(do\mid q)
			=
			\pi_{\rm old}^-(do\mid q).
		\end{align*}
	\end{proof}
	
	\begin{proof}[Proof of Proposition~\ref{proposition:1}]
		Let
		\[
		\rho_t(\theta)
		=
		\frac{\pi_\theta(o_t\mid q,o_{<t})}
		{\pi_{\rm old}(o_t\mid q,o_{<t})}.
		\]
		Removing the KL regularizer from Eq.~\eqref{eq:grpo-objective} gives
		\[
		\mathcal{J}_{\rm GRPO}^{\rm clip}(\theta)
		=
		\mathbb E_q
		\mathbb E_{o\sim\pi_{\rm old}(\cdot\mid q)}
		\left[
		\frac{1}{|o|}
		\sum_{t=1}^{|o|}
		f\left(\rho_t(\theta),A(o\mid q)\right)
		\right].
		\]
		For the clipping function in Eq.~\eqref{eq:grpo-clip},
		\[
		f(\rho,A(o\mid q))
		=
		A^+(o\mid q)\min(\rho,1+\epsilon)
		-
		A^-(o\mid q)\max(\rho,1-\epsilon).
		\]
		Averaging over tokens and using Eq.~\eqref{eq:score-positive-negative}, we obtain
		\begin{align}
			\mathcal{J}_{\rm GRPO}^{\rm clip}(\theta)
			&=
			\mathbb E_q
			\mathbb E_{o\sim\pi_{\rm old}(\cdot\mid q)}
			\left[
			A^+(o\mid q)s_\theta^+(o,q)
			-
			A^-(o\mid q)s_\theta^-(o,q)
			\right] \notag \\
			&=
			\mathbb E_q
			W(q)
			\left(
			\mathbb E_{o\sim\widetilde{\pi}_q^+}
			[s_\theta^+(o,q)]
			-
			\mathbb E_{o'\sim\widetilde{\pi}_q^-}
			[s_\theta^-(o',q)]
			\right) \notag \\
			&=
			\mathbb E_q
			W(q)
			\,
			\mathbb E_{o\sim\widetilde{\pi}_q^+,\,o'\sim\widetilde{\pi}_q^-}
			\left[
			s_\theta^+(o,q)-s_\theta^-(o',q)
			\right].
			\label{eq:appendix-grpo-weighted-discrimination}
		\end{align}
		By Lemma~\ref{lemma:advantage-reweighting},
		\[
		W(q)
		=
		\frac{1}{2}
		\mathbb E_{\bar o\sim\pi_{\rm old}(\cdot\mid q)}
		\big[
		|A(\bar o\mid q)|
		\big].
		\]
		Substituting this identity into Eq.~\eqref{eq:appendix-grpo-weighted-discrimination} yields
		\begin{align*}
			\mathcal{J}_{\rm GRPO}^{\rm clip}(\theta)
			=
			\mathbb E_q
			\left[
			\frac{1}{2}
			\mathbb E_{\bar o\sim\pi_{\rm old}(\cdot\mid q)}
			\big[
			|A(\bar o\mid q)|
			\big]
			\,
			\mathbb E_{o\sim\widetilde{\pi}_q^+,\,o'\sim\widetilde{\pi}_q^-}
			\big[
			s_\theta^+(o,q)-s_\theta^-(o',q)
			\big]
			\right],
		\end{align*}
		which is Eq.~\eqref{eq:grpo-general-reformulation}.
		
		For binary rewards with $0<p(q)<1$, applying Lemma~\ref{lemma:binary-specialization} to Eq.~\eqref{eq:appendix-grpo-weighted-discrimination} gives
		\begin{align*}
			\mathcal{J}_{\rm GRPO}^{\rm clip}(\theta)
			=
			\mathbb E_q
			\sqrt{p(q)(1-p(q))}
			\,
			\mathbb E_{o\sim\pi_{\rm old}^+(\cdot\mid q),\,o'\sim\pi_{\rm old}^-(\cdot\mid q)}
			\big[
			s_\theta^+(o,q)-s_\theta^-(o',q)
			\big],
		\end{align*}
		which is Eq.~\eqref{eq:grpo-binary-reformulation}.
	\end{proof}
	
	\begin{proof}[Proof of Lemma~\ref{lemma:grpo-gradient}]
		By Eq.~\eqref{eq:grpo-empirical-discriminative},
		\[
		\widehat{\mathcal J}_{\rm GRPO}^{\rm clip}(q)
		=
		\sqrt{\widehat p_q(1-\widehat p_q)}
		\left(
		\frac{1}{N^+}\sum_{i=1}^{N^+}s_\theta^+(o_i^+,q)
		-
		\frac{1}{N^-}\sum_{j=1}^{N^-}s_\theta^-(o_j^-,q)
		\right).
		\]
		Differentiating with respect to $s_\theta^+(o_i^+,q)$ and $s_\theta^-(o_j^-,q)$ gives
		\[
		\frac{\partial \widehat{\mathcal J}_{\rm GRPO}^{\rm clip}(q)}
		{\partial s_\theta^+(o_i^+,q)}
		=
		\frac{1}{N^+}
		\sqrt{\widehat p_q(1-\widehat p_q)},
		\qquad
		\frac{\partial \widehat{\mathcal J}_{\rm GRPO}^{\rm clip}(q)}
		{\partial s_\theta^-(o_j^-,q)}
		=
		-
		\frac{1}{N^-}
		\sqrt{\widehat p_q(1-\widehat p_q)}.
		\]
	\end{proof}
	
	\subsection{Proof of Lemma~\ref{lemma:conspo-gradient}}
	\label{app:conspo-gradient}
	
	\begin{proof}
		For brevity, write
		\[
		s_i^+=s_\theta(o_i^+,q),
		\qquad
		s_j^-=s_\theta(o_j^-,q),
		\qquad
		Z_i=\exp(s_i^+/\tau)+\sum_{k=1}^{N^-}\exp(s_k^-/\tau).
		\]
		Then Eq.~\eqref{eq:conspo-infonce} becomes
		\[
		\widehat{\mathcal J}_{\rm NCE}(q)
		=
		\frac{1}{N^+}
		\sum_{i=1}^{N^+}
		\tau
		\left(
		\frac{s_i^+}{\tau}
		-
		\log Z_i
		\right),
		\]
		and the contrastive probabilities in Eq.~\eqref{eq:conspo-prob} are
		\[
		P_i^+
		=
		\frac{\exp(s_i^+/\tau)}{Z_i},
		\qquad
		P_{ij}^-
		=
		\frac{\exp(s_j^-/\tau)}{Z_i}.
		\]
		
		For a positive rollout score $s_i^+$, only the $i$-th summand depends on it. Hence
		\begin{align*}
			\frac{\partial \widehat{\mathcal J}_{\rm NCE}(q)}
			{\partial s_i^+}
			&=
			\frac{1}{N^+}
			\tau
			\left(
			\frac{1}{\tau}
			-
			\frac{1}{Z_i}
			\frac{\partial Z_i}{\partial s_i^+}
			\right)
			=
			\frac{1}{N^+}
			\tau
			\left(
			\frac{1}{\tau}
			-
			\frac{1}{Z_i}
			\frac{\exp(s_i^+/\tau)}{\tau}
			\right)
			=
			\frac{1}{N^+}(1-P_i^+).
		\end{align*}
		For a negative rollout score $s_j^-$, all summands depend on it through their denominators, giving
		\begin{align*}
			\frac{\partial \widehat{\mathcal J}_{\rm NCE}(q)}
			{\partial s_j^-}
			&=
			-
			\frac{1}{N^+}
			\sum_{i=1}^{N^+}
			\tau
			\frac{1}{Z_i}
			\frac{\partial Z_i}{\partial s_j^-}
			=
			-
			\frac{1}{N^+}
			\sum_{i=1}^{N^+}
			\tau
			\frac{1}{Z_i}
			\frac{\exp(s_j^-/\tau)}{\tau}
			=
			-
			\frac{1}{N^+}
			\sum_{i=1}^{N^+}
			P_{ij}^-.
		\end{align*}
		This gives Eq.~\eqref{eq:conspo-gradient}.
		
		Using $P_i^+ + \sum_{j=1}^{N^-}P_{ij}^- = 1$, we obtain
		\begin{align*}
			\sum_{j=1}^{N^-}
			\frac{\partial \widehat{\mathcal J}_{\rm NCE}(q)}
			{\partial s_j^-}
			&=
			-
			\frac{1}{N^+}
			\sum_{i=1}^{N^+}
			\sum_{j=1}^{N^-}
			P_{ij}^-
			=
			-
			\frac{1}{N^+}
			\sum_{i=1}^{N^+}
			(1-P_i^+).
		\end{align*}
		Moreover,
		\begin{align*}
			\sum_{i=1}^{N^+}
			\frac{\partial \widehat{\mathcal J}_{\rm NCE}(q)}
			{\partial s_i^+}
			+
			\sum_{j=1}^{N^-}
			\frac{\partial \widehat{\mathcal J}_{\rm NCE}(q)}
			{\partial s_j^-}
			&=
			\frac{1}{N^+}
			\sum_{i=1}^{N^+}
			(1-P_i^+)
			-
			\frac{1}{N^+}
			\sum_{i=1}^{N^+}
			(1-P_i^+)
			=
			0.
		\end{align*}
		This gives Eq.~\eqref{eq:conspo-gradient-aggregate}.
		
		Finally, for any two negative rollouts $o_j^-$ and $o_k^-$,
		\begin{align*}
			P_{ij}^-
			&=
			\frac{\exp(s_j^-/\tau)}{Z_i}
			=
			\exp\left(
			\frac{s_j^- - s_k^-}{\tau}
			\right)
			\frac{\exp(s_k^-/\tau)}{Z_i}
			=
			\exp\left(
			\frac{s_j^- - s_k^-}{\tau}
			\right)
			P_{ik}^-.
		\end{align*}
		Summing over $i$ yields
		\[
		\sum_{i=1}^{N^+}P_{ij}^-
		=
		\exp\left(
		\frac{s_j^- - s_k^-}{\tau}
		\right)
		\sum_{i=1}^{N^+}P_{ik}^-.
		\]
		Combining this identity with Eq.~\eqref{eq:conspo-gradient} gives
		\[
		\frac{
			\left|
			\partial \widehat{\mathcal J}_{\rm NCE}(q) / \partial s_j^-
			\right|
		}{
			\left|
			\partial \widehat{\mathcal J}_{\rm NCE}(q) / \partial s_k^-
			\right|
		}
		=
		\frac{\sum_{i=1}^{N^+}P_{ij}^-}{\sum_{i=1}^{N^+}P_{ik}^-}
		=
		\exp\left(
		\frac{s_j^- - s_k^-}{\tau}
		\right).
		\]
		This gives Eq.~\eqref{eq:conspo-hard-negative}.
	\end{proof}
	\twocolumn
	\section{Detailed Related Work}
	\label{app:related}
	\paragraph{RLVR for LLM reasoning.}
	GRPO~\citep{guo2025deepseek} circumvents the reliance on a parameterized value model by estimating group-relative advantages directly from outcome-level rewards, maintaining robust performance across diverse reasoning tasks.
	Building on this paradigm, subsequent efforts have advanced GRPO toward better performance, more stable training, and more efficient reasoning.
	Dr.~GRPO~\citep{liu2025understanding} removes the standard deviation in advantage estimation to mitigate difficulty bias, and adjusts loss allocation to reduce length bias.
	DAPO~\citep{yu2025dapo} introduces dynamic sampling to reduce high-variance gradient updates, and employs length-based reward shaping for finer-grained length control.
	GMPO~\citep{zhao2026geometricmean} addresses outlier effects in asynchronous training by replacing token-level clipping with sequence-level geometric-mean clipping.
	CISPO~\citep{chen2025minimax} introduces soft clipping to preserve learning signals discarded by hard clipping, while SAPO~\citep{gao2025soft} applies asymmetric, temperature-controlled clipping for finer-grained and stable updates.
	HAPO~\citep{huang2026hapo} and ShorterBetter~\citep{yi2025shorterbetter} promote concise reasoning by selecting shorter trajectories from either historical or current generations.
	While DisCO~\citep{li2025disco} reinterprets GRPO through a discriminative lens under binary rewards and explores alternative scoring functions, it does not address the score-insensitive credit assignment induced by the GRPO objective.
	In contrast, ConSPO optimizes a group-wise InfoNCE-style objective for online RLVR, using likelihood-aligned sequence scores and a curriculum-scheduled margin to enable contrast-sensitive credit assignment within rollout groups.
	\paragraph{Contrastive Objectives for Policy Alignment.}
	Contrastive learning trains models by distinguishing positive samples from negative distractors, with InfoNCE-style objectives serving as a standard formulation~\citep{oord2018representation}.
	Foundational methods improve contrastive learning by enlarging the negative set, for example, through large-batch sampling in SimCLR~\citep{chen2020simple} or momentum queues in MoCo~\citep{he2020momentum}.
	Inspired by contrastive learning principles, recent policy alignment methods formulate preference optimization as a task of discriminating between preferred and dispreferred responses.
	For example, DPO~\citep{rafailov2023direct} and CPO~\citep{xu2024contrastive} optimize policy likelihoods over preference pairs without explicit reward modeling, while SimPO~\citep{meng2024simpo} further uses length-normalized sequence log-probabilities as an implicit ranking score.
	Unlike these offline preference optimization methods, ConSPO transfers the contrastive principle to online RLVR, where positive-negative comparisons are constructed from verifiable rewards on policy-generated rollouts rather than from static pairwise preference data.
	\begin{table}[t]
		\centering
		\small
		\setlength{\tabcolsep}{9pt}
		\renewcommand{\arraystretch}{1.08}
		\begin{tabular}{@{}>{\centering\arraybackslash\bfseries}p{0.25\linewidth} l c@{}}
			\toprule
			\textbf{Model} & \textbf{Config} & \textbf{Value} \\
			\midrule
			\multirow{10}{*}{\parbox{2.5cm}{\centering R1-Distill-Qwen\\\&\\R1-Distill-Llama}} 
			& max prompt length & 1024 \\
			& max response length & 8192 \\
			& rollout temperature & 1.0 \\
			& learning rate & $2 \times 10^{-6}$ \\
			& train batch size & 256 \\
			& number of rollouts & 8 \\
			& training epochs & 4 \\
			& contrastive temperature $\tau$ & 10 \\
			& maximum margin $M$ & 0.01 \\
			& margin warmup ratio $\alpha$ & 30\% \\
			\bottomrule
		\end{tabular}
		\caption{\textbf{Training configurations for the main experiments.}}
		\label{tab:implementation_config}
	\end{table}
	\begin{table}[t]
		\centering
		\small
		\setlength{\tabcolsep}{12pt}
		\renewcommand{\arraystretch}{1.08}
		\begin{tabular}{@{}>{\centering\arraybackslash\bfseries}p{0.25\linewidth} l c@{}}
			\toprule
			\textbf{Model} & \textbf{Config} & \textbf{Value} \\
			\midrule
			\multirow{4}{*}{\parbox{2.5cm}{\centering R1-Distill-Qwen\\\&\\R1-Distill-Llama}} 
			& max response length & 8192 \\
			& rollout temperature & 0.6 \\
			& top-$p$ & 0.95 \\
			& evaluation interval & 100 steps \\
			\bottomrule
		\end{tabular}
		\caption{\textbf{Evaluation configurations for the main experiments.}}
		\label{tab:evaluation_config}
	\end{table}
	\begin{table*}[t]
		\centering
		\small
		\setlength{\tabcolsep}{3.5pt}
		\renewcommand{\arraystretch}{1.0}
		\begin{tabularx}{\textwidth}{l*{8}{>{\centering\arraybackslash}X}}
			\toprule
			\textbf{Method}
			& \textbf{AIME24}
			& \textbf{AIME25}
			& \textbf{AIME26}
			& \textbf{HMMT25}
			& \textbf{MATH500}
			& \textbf{AMC}
			& \mbox{\textbf{O-Bench}}
			& \textbf{Avg.} \\
			\midrule
			GRPO
			& 60.7 & 43.1 & 49.6 & 26.0
			& 95.0 & 85.4 & 61.8 & 60.2 \\
			\midrule
			\rowcolor{green!15}
			\textbf{ConSPO (Ours)}
			& \textbf{65.0}
			& \textbf{46.3}
			& \textbf{52.5}
			& \textbf{27.6}
			& \textbf{95.8}
			& \textbf{86.0}
			& \textbf{67.6}
			& \textbf{63.0} \\
			\bottomrule
		\end{tabularx}
		\caption{
			\textbf{Additional results on DeepSeek-R1-Distill-Qwen-32B.}
			We report avg@32 on AIME, HMMT, and AMC, and pass@1 on the remaining benchmarks.
			O-Bench is short for OlympiadBench.
		}
		\label{tab:32b}
	\end{table*}
	\section{Detailed Experimental Settings}
	\subsection{Benchmarks}
	\label{Appendix: benchmark}
	We evaluate the mathematical reasoning ability of the trained models on seven benchmarks that are commonly used in recent RLVR studies~\citep{tan2026deepscaler,yan2026learning,zhang2026adhintadaptivehintsdifficulty,zhang2025detecting}.
	
	\begin{itemize}[leftmargin=*]
		\item \textbf{AIME24--26} --- Problems from the 2024, 2025, and 2026 American Invitational Mathematics Examination. AIME consists of challenging high-school competition problems with integer-valued answers, requiring precise multi-step reasoning rather than multiple-choice selection.
		
		\item \textbf{HMMT25} --- Problems from the February 2025 Harvard--MIT Mathematics Tournament. The benchmark covers competition-level problems in algebra and number theory, combinatorics, and geometry, as well as team-based rounds.
		
		\item \textbf{MATH500} --- A 500-problem evaluation subset of the MATH benchmark. It contains competition-style problems across seven subjects, including algebra, geometry, counting and probability, number theory, and precalculus.
		
		\item \textbf{AMC} --- Problems from the American Mathematics Competitions, covering a broad range of high-school mathematics topics. Although AMC is originally multiple-choice, the problems still require careful reasoning and symbolic manipulation.
		
		\item \textbf{OlympiadBench} --- An Olympiad-level benchmark containing challenging mathematics and physics problems. It evaluates advanced scientific reasoning on problems collected from Olympiad-style competitions and related examinations.
	\end{itemize}
	
	Following prior work~\citep{liu2025ghpo}, we report avg@32 on AIME, HMMT, and AMC to reduce the variance caused by their relatively small test sizes, and report pass@1 on the remaining benchmarks.
	\subsection{Implementation Details}
	\label{Appendix: implementation}
	The training and evaluation configurations for the R1-Distill models are summarized in Tables~\ref{tab:implementation_config} and~\ref{tab:evaluation_config}, respectively.
	For the experiments on Qwen3-4B-Base and Llama-3.2-3B-Instruct, we use the same settings as the R1-Distill models, except that the maximum response length is set to 4096 during both training and evaluation. In addition, Qwen3-4B-Base is trained for 2 epochs due to its faster convergence, while the other models are trained for 4 epochs.
	\section{Additional Experimental Results}
	\label{app:32b}
	To further evaluate the scalability of ConSPO, we conduct additional experiments on DeepSeek-R1-Distill-Qwen-32B. We use the same training data and hyperparameter settings as in the main experiments, and train both GRPO and ConSPO for two epochs. As shown in Table~\ref{tab:32b}, ConSPO outperforms GRPO across all seven benchmarks, improving the average score from 60.2 to 63.0. These results demonstrate that the effectiveness of ConSPO extends to substantially larger model scales.
	\begin{figure}[t]
		\centering
		\includegraphics[width=\linewidth]{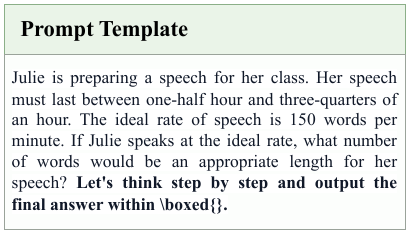}
		\caption{\textbf{Prompt template for rollout generation and evaluation.}}
		\label{fig:prompt_template}
	\end{figure}
	\section{Prompt Template}
	We use a unified prompt template for both rollout generation and evaluation across all methods.
	As shown in Figure~\ref{fig:prompt_template}, the prompt asks the model to solve the problem step by step and place the final answer within \verb|\boxed{}|, which facilitates answer extraction and rule-based verification.
	
\end{document}